\documentclass{ecai}  

\usepackage{graphicx}
\usepackage{latexsym}
\usepackage{graphicx}
\usepackage{latexsym}
\usepackage[american]{babel}
\usepackage{amssymb}
\usepackage{algorithm,algcompatible,amsmath}
\usepackage{graphicx}
\usepackage{subcaption}
\usepackage{mwe}
\usepackage{wrapfig}
\usepackage{amsmath}
\usepackage{mathtools}
\usepackage{algorithmicx}
\usepackage[utf8]{inputenc}
\usepackage{array, xcolor, lipsum, bibentry,fancyhdr}
\usepackage{booktabs}
\usepackage[tableposition=top]{caption} 
\usepackage{adjustbox}
\usepackage{booktabs}
\DeclareMathOperator{\diag}{diag}
\makeatletter
\newcommand{\thickhline}{%
    \noalign {\ifnum 0=`}\fi \hrule height 2pt
    \futurelet \reserved@a \@xhline
}
\newcolumntype{"}{@{\hskip\tabcolsep\vrule width 1pt\hskip\tabcolsep}}
\makeatother

\newcommand{\matZ}{\mathbf{Z}}

\newcommand{\vecz}{\mathbf{z}}

\newcommand{\vecx}{\mathbf{x}}
\newcommand{\vecy}{\mathbf{y}}

\newcommand{\vecl}{\mathbf{\ell}}

\newtheorem{proof}{Proof}

\newtheorem{lemma}{Lemma}

\DeclareMathOperator*{\argmin}{arg\,min}

\begin{document}

\begin{frontmatter}

\title{CLIMAX: An exploration of Classifier-Based Contrastive Explanations \vspace{-1mm}}

\author[A]{\fnms{Praharsh}~\snm{Nanavati}
\thanks{Corresponding Author Email: praharsh19@iiserb.ac.in.}}
\author[B]{\fnms{Ranjitha}~\snm{Prasad}} 

\address[A]{Indian Institute of Science Education and Research, Bhopal}
\address[B]{Indraprastha Institute of Information Technology, Delhi}

\begin{abstract}
Explainable AI is an evolving area that deals with understanding the decision making of machine learning models so that these models are more transparent, accountable, and understandable for humans. In particular, post-hoc model-agnostic interpretable AI techniques explain the decisions of a black-box ML model for a single instance locally, without the knowledge of the intrinsic nature of the ML model. Despite their simplicity and capability in providing valuable insights, existing approaches fail to deliver consistent and reliable explanations. Moreover, in the context of black-box classifiers, existing approaches justify the predicted class, but these methods do not ensure that the explanation scores strongly differ as compared to those of another class. In this work we propose a novel post-hoc model agnostic XAI technique that provides contrastive explanations justifying the classification of a black box classifier along with a  reasoning as to why another class was not predicted. Our method, which we refer to as CLIMAX which is short for Contrastive Label-aware Influence-based Model Agnostic XAI, is  based on local classifiers . In order to ensure model fidelity of the explainer, we require the perturbations to be such that it leads to a class-balanced surrogate dataset. Towards this, we employ a label-aware surrogate data generation method based on random oversampling and Gaussian Mixture Model sampling. Further, we propose influence subsampling in order to retaining effective samples and hence ensure sample complexity. We show that we achieve better consistency as compared to baselines such as LIME, BayLIME, and SLIME. We also depict results on textual and image based datasets, where we generate contrastive explanations for any black-box classification model where one is able to only query the class probabilities for an instance of interest. 
\end{abstract}

\end{frontmatter}

\section{Introduction}
As AI technology deployment is increasing especially in safety-critical domains, it has necessitated that ML models be interpretable and trustworthy while being accurate. Trustworthiness of an AI system is possible if the target users understand the \emph{how and why} about ML model predictions. Interpretability is also essential owing to severe biases that are induced in the decision-making process of  deep neural networks (DNNs) when  subject to adversaries \cite{social_attack, OnepixelFooling, healthcare_attack}. Governments across the world have introduced regulations towards the ethical use of AI. For instance, General Data Protection Regulation (GDPR) passed in Europe requires businesses to provide understandable justifications to their users for decisions of AI systems that directly affect them \cite{adadi2018peeking}. 

Popular categorization of existing XAI methods is based on XAI models being local \cite{LIME, SHAP, ALIME, DLIME, OptiLIME,Unravel} or global \cite{SHAP}, model agnostic \cite{LIME, SHAP} or model specific \cite{Deeplift}, in-hoc or post-hoc, perturbation or saliency-based \cite{ALIME}, concept-based or feature-based \cite{adadi2018peeking}, etc.
The simplest among them is the well-established post-hoc, perturbation-based techniques such as LIME \cite{LIME} and KernelSHAP \cite{SHAP}. Perturbation-based post hoc explainers offer a model agnostic means of interpreting black-box ML models  while requiring query-level access for a single instance. These methods define a data generating process to obtain weighted perturbations (surrogate data) in the neighborhood of the index sample, and subsequently employ easy-to-explain linear regression  model to obtain per-feature importance weights. Despite the widespread usage of these techniques, subsequent works have pointed out various issues. For instance, LIME leads to inconsistent explanations on a given sample \cite{Indices, BayLIME, DLIME, ALIME, OptiLIME}, hampers its use in safety-critical systems. Although KernelSHAP partially counters the stability issue, it employs training data for explanations. However, more importantly, these methods use  feature attribution to explain the prediction of a black-box models and do not produce contrastive explanations. 

More recently, contrastive \cite{dhurandhar2018explanations, MACEM} and counterfactual approaches \cite{verma2020counterfactual} have been proposed. The goal of a contrastive explanation is not only to justify the output class of an input, but also what should be absent to maintain the original classification, while counterfactual explanations specify necessary minimal changes in the input so that an alternate output is obtained. In this work, we are interested in label-aware, post-hoc technique for providing model agnostic contrastive explanations in the locality of a given instance (which we refer to as the index sample).

\begin{figure*}[h]
\centering
\includegraphics[width=0.85\textwidth]{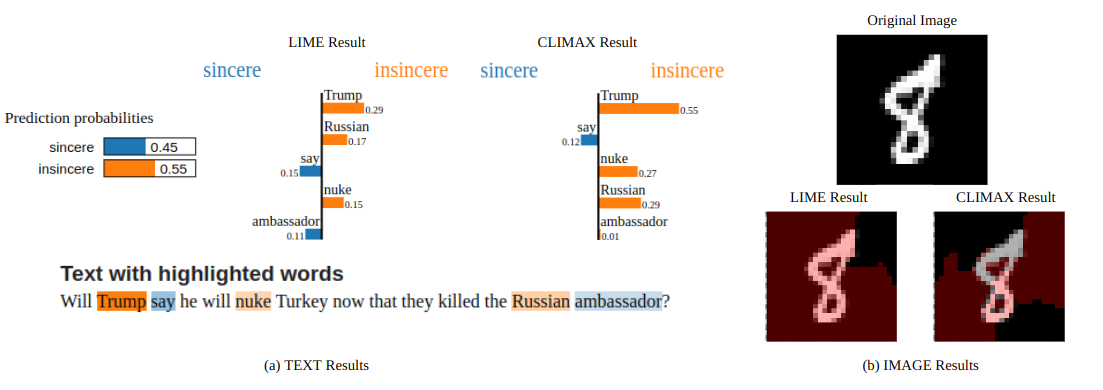}
\caption{Comparisons between CLIMAX and LIME for an instance of the (a) Quora Insincerity Dataset, and (b) The MNIST dataset. The results of CLIMAX are more contrastive as compared to LIME. In the case of textual results we can clearly see the confident nature of CLIMEX, and for the case of Images, we can see that CLIMAX provides more reliable results, as it gives a more precise region, and clearly highlights the ambiguous regions. The same region may play a huge positive role in the other digits. This has been portrayed later.}
\label{fig:Text_Result_CLIMAX}
\end{figure*}

Studies in philosophy and
social science point out that in general, humans prefer contrastive explanations \cite{Lipton}. Let us suppose that the predicted class of the black-box model for the $i$-th instance is $c_i$, and the alternative class-label is $c_{-i}$. Here, answering the question, \emph{Why $c_i$?} leads to just explaining the predicted class as done in most of the post-hoc model agnostic techniques such as LIME, BayLIME, Unravel   and KernelSHAP. However, it is natural to seek a contrastive explanation where queries are of the form \emph{why $c_i$ and not $c_{-i}$?}. As pointed out in \cite{dhurandhar2018explanations},
contrastive explanations  highlight what is minimally sufficient in an input to justify its classification, and identify contrastive features that should be minimally present and critically absent to distinguish it from another input is seemingly close but would be classified differently. Most of the available contrastive explainers require the original training samples, are model-aware, or use complex data generation procedure which leads to opacity in explainer models  \cite{wang2022not, dhurandhar2018explanations}. 

Alternately, we propose a contrastive explainer which is model-agnostic and perturbation-based. In the context of a classification based black-box model, a regression based explanation model provides explanations based on the surrogate dataset generated using the pre-defined data generating process. Note that the data generating process does not mandate samples from all classes since a regression based explainer does not require a balanced dataset.  Essentially, this implies that the class-based feature attribution scores are provided when there may be no information about this class in the surrogate dataset. We question the basic paradigm in post-hoc perturbation based methods which advocates the use a local linear regression model and instead, we focus on a local logistic regression model. A classifier based explanation model necessitates that the surrogate samples form a balanced dataset, i.e., there are approximately equal number of samples from different classes. Essentially this implies that class-based attribution score is obtained after ensuring that surrogate data samples with all class information is present in the surrogate dataset. This leads to contrastive explanations and improved stability of the explainer method. 

\paragraph{Contributions:} In this work, we propose a contrastive label-aware sample-efficient post-hoc explainable AI (XAI) technique called CLIMAX. Briefly, our contributions are as follows:
\begin{itemize}
    \item We propose two variants of the logistic regression (LR) based explainer and generation of a label-wise balanced surrogate dataset. Similar to LIME, the per-feature weight obtained from the LR model provides the contrastive feature attribution scores. Essentially this allows us to exploit the classification boundary of the black-box model and explain each instance, from a dual point of view, i.e.,
\begin{itemize}
    \item \textit{Why point `a' must lie in class ${c}_{i}$} and 
    \item \textit{Why point `a' must not lie in classes ${c}_{-i}$}
\end{itemize}
    \item Influence functions are a classic technique from robust statistics which trace a model’s prediction through the learning algorithm and back to its training data thereby identifying training points most responible for a given prediction. We use this module within our surrogate data generator as it helps reduce the sample complexity. We observe that the performance of the model after subsampling stays at par with the original model, and sometimes even surpasses it.
\end{itemize} 

\section{Related Works and Novelty}

In this work, we are interested in label-aware model-agnostic post-hoc locally interpretable models. We discuss the related works by highlighting the critical aspects in comparison to the proposed method as in the sequel.\\
\textbf{Instability}:~ Instability or inconsistency issues in the explanations scores of LIME over several iterations \cite{DLIME,ALIME} is well-known. This inconsistency occurs due to random perturbation-based surrogate datasets. A deterministic hierarchical clustering approach for consistent explanations was proposed in DLIME \cite{DLIME}, and its biggest drawback is that it requires training data for clustering. To avoid the additional task of `explaining the explainer,' techniques like ALIME \cite{ALIME,SmartSamplingGAN} are not preferred. A parametric Bayesian method as proposed by \cite{BayLIME}, where a weighted sum of the prior knowledge and the estimates based on new samples obtained from LIME is used to get explanations in a Bayesian linear regression framework. Both LIME and BayLIME employ hyperparameters (kernel-width) that need to be tuned. Recently, \cite{slack2020reliable} proposed a technique known as focused sampling, which utilizes uncertainty estimates of explanation scores. In \cite{Unravel}, authors propose a Gaussian process based active learning module to abate issues of instability. In this work, we used a sampling strategy that ensures that we are balanced with respect to the labels. This ensures that we obtain good stability and low inconsistency in explanation scores. \\
\textbf{Sample Complexity}:~Sample efficiency in post-hoc models is a crucial factor in efficiently obtaining reliable explanations, and there is consensus in the community that explainable models must use as few samples for an explanation as possible \cite{slack2020reliable}. Approaches such as LIME, KernelSHAP, and BayLIME do not provide any guidance on choosing the number of perturbations, although this issue has been acknowledged \cite{BayLIME}. Influence functions \cite{koh2017understanding} are known to reduce the sampling complexity and the reduced sample set can be used for providing robust explanations. We exploit influence functions to achieve fidelity and sample complexity goals simultaneously via our surrogate dataset.\\
\textbf{Classifier-based Explainers}: ~Techniques like LIME \cite{LIME} and KernelSHAP \cite{SHAP} fit linear regression model on classification probabilities. This leads to a separate set of explanation scores for each class, where the scores try to explain why a class is predicted. Intuitively, regression black-box models are well-explained by linear regression explanation models and black-box classifiers are better explained by linear classifier explainers. In particular, classifier based explainers are expected to provide a  robust set of explanations as they can exploit the classification boundary explicitly to provide information about why a point lies in a class $c_i$, and why not in the other classes, $c_{-i}$. This problem has been acknowledged in \cite{moradi2021post}, and the authors propose explanation scores based on confident item sets. In \cite{vlassopoulos2020explaining}, authors approximate the local decision boundary, but use a variational autoencoder for surrogate data generation, leading to opaque data generation. We propose a classifier based explainer which makes use of probabilities of all classes, and hence is more contrastive.\\
\textbf{Constrastive Explainers}:
Contrastive explanations clarify why an event occurred in contrast to another. They are inherently intuitive to humans to both produce and comprehend. There are a few techniques that already exist in the literature, such as the Contrastive Explanations Method, which makes use of pertinent positives and pertinent negatives to define those features are important and those that are not, respectively \cite{dwivedi2023explainable, jacovi2021contrastive, yang2022omnixai, dhurandhar2018explanations}. In \cite{wang2022not}, authors propose a framework that convert an existing back-propagation explanation method to build class-contrastive explanations, especially in the context of DNNs. However, these methods are not model agnostic, and often assume  access to training data. In \cite{rathi2019generating}, authors repurpose Shapley values to generate counterfactual and contrastive global explanations. In \cite{MACEM}, authors propose Model Agnostic Contrastive Explanations Method (MACEM), to generate contrastive explanations for any classification model where one is able to only query the class probabilities for a desired input restricted to be structured tabular data. \\
\textbf{Novelty:} In comparison, CLIMAX is novel in the following ways:
\begin{itemize}
    \item CLIMAX provides feature importances by explaining as to why the index sample belongs to a specific class and in the process, it also provides strong justification about why other classes were not predicted. This effect is brought about in CLIMAX using local classifiers for explanations,  without  explicitly solving for pertinent positives and negatives.
    \item CLIMAX is a perturbation-based technique, which implies that it does not require any access to training data. 
    \item CLIMAX explains the decision boundary of the black-box classifier, which is the most relevant characteristic of classifiers that are optimized for accuracy.
\end{itemize}

\section{Mathematical Preliminaries}
\label{sec:math}
In this section, we describe the mathematical preliminaries of the popular local explainer namely LIME, for classifier models. 

Local explainer models are interpretable models that are used to explain individual predictions of black box machine learning models. Among several methods, Local interpretable model-agnostic explanations (LIME) portrays a concrete implementation of local explainer model. These models are trained to approximate the predictions of the underlying black box model locally, in the neighborhood of the sample of interest and hence, these models may or may not be a valid explainer globally.\cite{molnar2020interpretable}

\textbf{Notation} Let $f: \mathbb{R}^{d} \rightarrow [0, 1]$ denote a black-box binary classifier, that takes a data point $\vecx \in \mathbb{R}^d$ ($d$ features) and returns the probability that $\vecx$ belongs to a certain class. Our goal is to explain individual
predictions of $f$ locally. Let $\mathcal{Z}$ be a set of $n'$ randomly sampled instances
(perturbations) around $\vecx$. The proximity between $\vecx$ and any $\vecz \in \mathcal{Z}$ is given by $\pi_{\vecx}(\vecz) \in \mathbb{R}$. We denote the vector of these distances over the $n'$ perturbations in $\mathcal{Z}$ as $\Pi_{\vecx}(\matZ) \in \mathbb{R}^{n'}$. Let $\boldsymbol{\phi} \in \mathbb{R}^{d}$ denote the explanation in terms of feature importances for the prediction
$f(\vecx)$.

Let $\boldsymbol{y}_1, \boldsymbol{y}_0 \in \mathbb{R}^{n'}$ be the black-box predictions for $n'$ surrogate samples corresponding to class-1 and class-0, respectively, such that for the $i$-th instance in $\mathcal{Z}$, ${y}_1(i) = f(\vecz_i)$ and ${y}_0(i) = 1- f(\vecz_i)$, and since they are probabilities, ${y}_1(i),y_0(i) \in [0, 1]$. LIME explains the predictions of the classifier $f$ by learning a linear model locally around each prediction. Hence, in the case of LIME the coefficients of the linear model are assigned as $\boldsymbol{\phi}$ are treated as the feature contributions to the black box prediction \cite{slack2020reliable}. Accordingly, the objective function for LIME constructs an explanation that approximates the behavior of the black box accurately in the vicinity (neighborhood) of $\vecx$ by solving:
\begin{equation}
\argmin_{\boldsymbol{\phi}}\sum_{\vecz \in \mathcal{Z}} [f(\vecz) - \boldsymbol{\phi}^{T}\vecz]^2 \pi_{\vecx}(\vecz),
    \label{eq: objective}
\end{equation}
which has a closed-form solution for class $c \in \{0,1\}$ given by:
\begin{equation}
\boldsymbol{\hat{\phi}}_c = \matZ^{T}\diag(\Pi_{\vecx}(\matZ))\matZ + \boldsymbol{I})^{-1}(\matZ^{T}\diag(\Pi_{\vecx}(\matZ))\boldsymbol{y}_c.
\label{eq: closed-form}
\end{equation}
LIME assigns different importance scores to different classes as by design, it is not possible to incorporate the information about probabilities of both the classes into a single linear regression framework. As mentioned earlier, this is sufficient until the question is `why $c$', as this question does not seek explanations about the other classes. Furthermore, the challenge in LIME arises in selecting a valid neighborhood or locality for surrogate sampling. LIME uses random sampling where these samples are chosen heuristically: $\pi_{\vecx}(\vecz)$ is computed as the cosine or $l_2$ distance. 

\section{Proposed Techniques and Algorithms}
We propose a classifier-based explainer, which we refer to as Contrastive Label-aware Influence-based Model-Agnostic XAI (CLIMAX), to understand and exploit the classification boundary as dictated by the black-box model so as to explain each instance, from dual points of view as stated before. Essentially, our method's reasoning is based on why a given point  must lie in class $c_{i}$ and not in classes $c_{-i}$. This is possible as unlike LIME, at the time of assigning scores, CLIMAX has access to all the class probabilities and the local classifier fits its boundary according to that. 

CLIMAX explains the predictions of the binary classifier $f(\cdot)$ by learning a logistic regression model locally around each prediction where, the probability of class-1 and class-0 according to the explainer is given by $\sigma(\boldsymbol{\phi^{T}}\vecz)$ and $(1 - \sigma(\boldsymbol{\phi^{T}}\vecz))$, respectively, where $\sigma(\cdot)$ is the sigmoid function. We now define two different variants of the CLIMAX method.


\subsection{L-CLIMAX}
In this section, we propose a local classifier explainer that results in logistic outputs, and we formally refer to this as Logistic CLIMAX, or L-CLIMAX. In order to derive the loss function, we state the following lemma. 
\begin{lemma}
Given a dataset $\mathcal{D}$ with the $i$-th instance such that $\{ \vecz_i, \vecy_i\} \in \mathcal{D}$ where $\vecz_i \in \mathbb{R}^d$ are the covariates and $\vecy_i \in \mathbb{R}^{|\mathcal{C}|}$ represents the class-probabilities, linear model on logistic outputs can be obtained as
\begin{equation}
    \argmin_{\boldsymbol{\phi}}(\vecl - \boldsymbol{\phi}^{T}\matZ)^T 
\diag(\Pi_\vecx{(\matZ)})
(\vecl-\boldsymbol{\phi}^{T}\matZ),
\end{equation}
where the $i$-th entry of $\vecl \in \mathbb{R}^{n'}$  is given by $\vecl(i) = \log\left(\frac{y_i}{1-y_i}\right)$ obtained from the black-box model, the $i$-th column in $\matZ \in \mathbb{R}^{d \times n'}$ is given by the surrogate sample $\vecz_i$, and $\diag(\Pi_\vecx{(\matZ)})$ is a diagonal matrix whose $(i,i)$-th entry is given by  $\pi_\vecx(\vecz_i)$.
\end{lemma}
\begin{proof}
The output of the logistic explainer model is given as 
\begin{align}
    y_{i} = \sigma(\boldsymbol{\phi}^{T}\vecz_{i}) = \frac{1}{1+e^{-\boldsymbol{\phi}^{T}\vecz_{i}}}.
\end{align}
The above expression can be rewritten in terms of the \textit{log-odds} representation of the logistic output as
\begin{align}\label{eq:log}
  \boldsymbol{\phi}^{T}\vecz_{i} = \log\left(\frac{y_{i}}{1-y_{i}}\right) \triangleq \ell(\vecz_i).
\end{align}
The above formulation allows us to model the black box prediction of each perturbation $\vecz_i$ as a linear
combination of the corresponding feature values ($\boldsymbol{\phi}^{T}\vecz_{i}$) plus an error term, i.e.,
\begin{align}
   \ell(\vecz_i) = \boldsymbol{\phi}^{T}\vecz_{i} + \epsilon_i, 
   \label{eq:linreglogreg}
\end{align}
where we model $\epsilon_i \sim \mathcal{N}(0,\sigma^2)$. Here, $l(\vecz_i)$ is obtained from the black box classifier. Incorporating the objective function of LIME in the context of \eqref{eq:linreglogreg} leads to 
\begin{equation}
\argmin_{\boldsymbol{\phi}}\sum_{\vecz_i \in \matZ} [\ell(\vecz_i) - \boldsymbol{\phi}^{T}\vecz_i]^2 \pi_{\vecx}(\vecz_i).
   \label{eq: climax_objective}
\end{equation}
Rewriting \eqref{eq: climax_objective} in terms of vector and matrices namely $\vecl$, $\matZ$ and $\diag(\Pi_\vecx{(\matZ)})$, we obtain
\begin{equation}
    \argmin_{\boldsymbol{\phi}}(\vecl - \boldsymbol{\phi}^{T}\matZ)^T 
\diag(\Pi_\vecx{(\matZ)})
(\vecl-\boldsymbol{\phi}^{T}\matZ).
\label{eq:ClimaxObjLem1}
\end{equation}
\end{proof}

Solving \ref{eq:ClimaxObjLem1} by including a regularizer of the form $\lambda \Vert{\boldsymbol{\phi}}\Vert_{2}^{2}$, we obtain the closed form solution for $\boldsymbol{\phi}$ as
\begin{equation}
    \boldsymbol{\hat{\phi}} = ({\matZ \diag(\Pi_\vecx{(\matZ)}) \matZ^{T} + \lambda \boldsymbol{I}}) ^{-1} \matZ \diag(\Pi_\vecx{(\matZ)}) \vecl
\end{equation}

\subsection{CE-CLIMAX}
The second variant of CLIMAX constructs an explanation that approximates the behavior of the black box accurately in the vicinity of the index sample (neighborhood) of $\vecx_0$ by directly optimizing the log-loss, i.e., we obtain feature importance values by solving the following: 
\begin{equation}
\argmin_{\boldsymbol{\phi}}\sum_{\vecz_i \in \matZ} f(\vecz_i)\log \vecy_i+ (1-f(\vecz_i))\log{(1-y_i)},
\label{eq:CEObjective}
\end{equation}
where $y_i = \sigma(\boldsymbol{\phi}^{T}\vecz_i)$. We call this variant as Cross-Entropy CLIMAX or CE-CLIMAX.
Note that unlike LIME, we do not explicitly weigh each surrogate sample using $\pi_\vecx(\vecz)$ in the second variant. 

Some of the salient aspects of both of the above formulation as compared to LIME are as follows:
\begin{itemize}
    \item LIME-like methods that ask the question `why $c_i$?' provide explanations label-wise, i.e., they iterate over all labels, explain why the sample index should be a part of that class and provides the scores. L-CLIMAX and CE-CLIMAX iterate over two sets of probabilities, one corresponding to the current class label and the other corresponding to all probabilities of the remaining classes. This can be explicitly seen in the objective where we use both $\vecy_i$ and $1-\vecy_i$ together as in \ref{eq: climax_objective} and \eqref{eq:CEObjective}. 
    \item Interpretation of $\boldsymbol{\phi}$: In the case of LIME, $\boldsymbol{\phi}$ determines the feature importances according to the regressor values. In CLIMAX, $\boldsymbol{\phi}$ has a slightly different interpretation. Here, $\boldsymbol{\phi}$ is larger for those features that help in increasing the `contrast' between explanations. Nevertheless, in both LIME and CLIMAX, it is safe to say that $\boldsymbol{\phi}$ highlights important features.
    \item Both the above formulation has the simplicity and elegance of LIME and related methods. It remains training data agnostic and model-agnostic. Additionally, L-CLIMAX can be implemented with a slight change to the existing LIME framework.
\end{itemize}

\subsection{Imbalance-aware Surrogate Sampling}
An important aspect for realization of L-CLIMAX and CE-CLIMAX is surrogate sampling required to form the the set $\mathcal{Z}$ in the previous subsection. In order to ensure fidelity of explainers, our sampling technique needs to be imbalance-aware since we use classifier based local explainers.

We use the Bootstrap sampling technique, where we repeatedly sample the neighborhood of $\vecx$ with replacement. The main goal is to ensure that the surrogate dataset is balanced, i.e., it consists of atleast a few samples belonging to all classes albeit in different proportions. To achieve this, we perform Gaussian sampling similar to \cite{LIME} and increase the standard deviation appropriately (to increase the neighborhood size) to obtain surrogate instances from all classes. In order to further reduce  imbalance in $\mathcal{Z}$, we do the following:
\begin{itemize}
    \item \textbf{Random oversampling}: We oversample within the minority class in order to ensure that the classes are perfectly balanced. 
    \item \textbf{Gaussian Mixture models}: A Gaussian mixture model (GMM) is a probabilistic model that assumes all the instances are generated from a mixture of a finite number of Gaussian distributions with unknown parameters \cite{sklearn}. We train a GMM consisting of $c$ Gaussians using the bootstrapped samples, and later use it to appropriately oversample the minority classes to obtain a balanced surrogate dataset. 
\end{itemize}

\begin{algorithm}
 \caption{GMM -- Sampling from a Gaussian Mixture Model}.
 \label{alg:GMM}
 \begin{algorithmic}[1]
 \REQUIRE Imbalanced Surrogate dataset $\mathcal{D}$, and corresponding labels $\vecy_{\mathcal{D}}$, Number of classes $c$
\STATE Fit a GMM on $\mathcal{D}$ to get cluster mean and variances.
 \STATE Identify minority classes, based on the number of instances in each cluster.
 \STATE Sample the required number of minority class instances.
\ENSURE Oversampled Data from the Gaussian Mixture Model
\end{algorithmic} 
\end{algorithm}

The above detailed sampling strategies may not necessarily improve the quality of samples. However, diminishing the imbalance in $\mathcal{Z}$  helps us in maintaining local fidelity and a consistent contrastive nature in the  explanation scores $\boldsymbol{\phi}$. In the sequel, we also demonstrate the improved stability performance of CLIMAX as compared to other perturbation based methods. Subsequently, we perform forward feature selection as proposed in LIME, to obtain the top $k$ features, and then return the scores for the explanation. We have explained the entire algorithm in Algorithm~\ref{alg:climax}.

\begin{algorithm}
 \caption{CLIMAX -- Contrastive Label-aware Influence-based Model-Agnostic XAI method}.
 \label{alg:climax}
 \begin{algorithmic}[1]
  \REQUIRE Black-box binary classifier model $f$, Instance $\vecx \in \mathbb{R}^{d}$, Number of features $d$, Number of surrogate samples $n'$
 \STATE Using $(x_0, f_{p})$, generate $n'$ surrogate samples and $\vecy_{i} = f_{p}(\vecx_i) \forall i = [n']$. 
\STATE Train explainer model $f_{e}$, using the surrogate dataset $\mathcal{D}$ and $\vecy_{i}$ obtained in $1$. 
\IF{Surrogate Sampling Style = `GMM'}
\STATE Perform GMM Oversampling as mentioned in Algorithm \ref{alg:GMM}
\ELSE
\STATE Identify the Minority Classes, $c_{m}$ and perform Random Oversampling for all $c_m$
\ENDIF{}
\IF{Influence Subsampling is True}
\STATE Perform influence subsampling using \eqref{computePhi} and \eqref{computeP}.
\ENDIF{}
\STATE Fit the logistic regression model in the locality of $\vecx_0$ according to \eqref{eq:ClimaxObjLem1} or \eqref{eq:CEObjective}.
\STATE \textbf{return } Feature Importance Scores $\boldsymbol{\phi}$
\end{algorithmic} 
\end{algorithm}

\subsection{Sample Complexity}

Sample efficiency in post-hoc models is a crucial factor in obtaining reliable explanations, and there is consensus in the research community that explainable models must use as few samples for an explanation as possible \cite{slack2020reliable}.  In both variants of Climax, we oversample the surrogate samples in order to ensure a balanced surrogate dataset, and hence, we have some redundant information within the data. Approaches such as LIME, KernelSHAP, and BayLIME do not provide any guidance on choosing the number of perturbations, although this issue has been acknowledged in \cite{BayLIME}. In \cite{Unravel}, sample complexity is dictated by an acquisition function and sampling is achieved via Gaussian processes. Often, such methods turn out to be too complex. 

We consider subsampling the surrogate samples using influence functions. Rooted in statistics, influence functions estimate how the model parameters change when a data point is upweighted by a small amount $\boldsymbol{\epsilon}$. Using influence functions, Koh and Liang \cite{koh2017understanding} proposed a method for estimating the impact of removing a data point
from the training set (reducing its weight to $0$) on the model parameters. We use this method to perform subsampling within our surrogate dataset to improve its quality. Influence functions help to build a tool to quantify each data point’s quality, thereby keeping good examples and dropping bad examples to improve the model’s generalization ability. Previous works focus on weighted subsampling, that is, trying to maintain the model performance when dropping several data. The steps in the case of influence subsampling \cite{wang2020less} is as follows:
\begin{itemize}
    \item Train the explainer model on the full set of surrogate samples:
    \begin{align}
       \boldsymbol{\hat{\theta}} = \underset{\boldsymbol{\theta} \in \boldsymbol{\Theta}}{\operatorname{argmin}}\frac{1}{n}\sum_{i=1}^{n}L(\vecz_{i}, \boldsymbol{\theta}),
    \end{align}
    where $\hat{\theta}$, is the set of optimal parameters.
    \item Compute the influence function for each surrogate sample:
    \begin{equation}
        \boldsymbol{{\rho}} = (\boldsymbol{\rho}(\vecz_{1},\boldsymbol{\hat{\theta}}), \boldsymbol{\rho}(\vecz_{2},\boldsymbol{\hat{\theta}}),..,\boldsymbol{\rho}(\vecz_{n},\boldsymbol{\hat{\theta}})).
        \label{computePhi}
    \end{equation}
    Here, $\rho$ denotes the value of the influence function \cite{koh2017understanding}.
    \item Compute the sampling probability of each surrogate sample:
      \begin{equation}
        \boldsymbol{{\psi}} = (\boldsymbol{\psi(\vecz}_{1},\boldsymbol{\hat{\theta}}), \boldsymbol{\psi(\vecz}_{2},\boldsymbol{\hat{\theta}}),..,\boldsymbol{\psi(\vecz}_{n},\boldsymbol{\hat{\theta}})),
        \label{computeP}
    \end{equation}
    where $\psi$ denotes the sampling probability of each surrogate sample as computed in \cite{koh2017understanding}. Using these quantities, we obtain how influential a point is, and we can trim our surrogate dataset.
    \item Finally, we perform subsampling based on the influence scores and train a subset model using the reduced set of surrogate samples.
    $$\boldsymbol{\Tilde{\theta}} = \underset{\boldsymbol{\theta} \in \boldsymbol{\Theta}}{\operatorname{argmin}}\frac{1}{\{i,\boldsymbol{o}_{i}=1\}}\sum_{\boldsymbol{o}_{i}=1}L(\vecz_{i}, \boldsymbol{\theta})$$
    Where, $\tilde{\boldsymbol{\theta}}$ gives the optimal parameters for the subsampled sets and the function $\textbf{\textit{o}}$ is an indicator function which is $1$ if the $i^{th}$ point is included in the subsampled set or not.
\end{itemize}

\section{Results and Discussions}
In this section, we demonstrate the efficacy of the proposed CLIMAX framework on publicly-available datasets. In particular, we are interested in establishing the contrastive capability of CLIMAX and investigating the attributes such as stability (consistency in repeated explanations) and sample efficiency. We employ tabular(structured data), textual, and image datasets, and consider different black-box models for an explanation. \footnote{Source code available at \url{https://github.com/niftynans/CLIMAX}}
\vspace{-3mm}
\subsection{Datasets and Pre-processing}
We chose four distinct datasets from the UCI Machine Learning
repository \cite{Dua:2019} as well as Scikit-Learn \cite{sklearn} for the tabular data based experiments owing to their usage in the relevant literature
for classification based  prediction tasks. The description of the tabular 
datasets is as follows:
\begin{itemize}
    \item \textbf{Breast Cancer}: This dataset consists of $569$ instances, with $30$ features computed from an image of a breast mass, describing characteristics of the cell nuclei \cite{scikit-learn}. Hence, the classification task is to predict if the cancer is malignant or not. 
    
    \item \textbf{Parkinson's}: The Parkinson's classification dataset consists of $195$ instances of patients suffering and free from Parkinson's disease \cite{Dua:2019}. With $22$ unique features per recording, the task is to classify whether a given patient has Parkinson's or not.

    \item \textbf{Ionosphere}: This dataset consists of $34$ features, and $351$ instances of radar data that was collected in Goose Bay, Labrador \cite{Dua:2019}. The targets were free electrons in the ionosphere. The classification task was to label the instances as `good' or `bad'. 
    
    \item \textbf{Diabetes}:  This dataset consists of $8$ attributes and $768$ data points that describes the medical records for Diabetes patients. It contains information about the pregnancy status, insulin levels, blood pressure and other medical attributes about the patients \cite{Dua:2019}.
\end{itemize}

For text, we use the Quora Insincere Questions dataset \cite{kaggle}, where the classification task is to identify if a question is sincere or not. We  also use the  20 News Groups dataset \cite{scikit-learn} where the classification task is amongst two groups: Atheism and Christianity. We determine whether a given paragraph is written by an atheist or a Christian. Due to lack of space, we present the results for the 20News Groups dataset in the Supplementary.


For images, we use the MNIST dataset \cite{deng2012mnist} in order to contrast the relevant regions that contribute to the prediction of each digit.


\subsection{Baselines}
CLIMAX focuses only on classification-based tasks. It is a perturbation-based technique, i.e., we do not assume any knowledge of the training samples or an autoencoder that may be trained on original data, but instead, we obtain surrogate samples in the vicinity of the index sample. Hence, we baseline CLIMAX using other perturbation-based methods that employ similar assumptions in their workflow. We use LIME and BayLIME \cite{BayLIME} as our baselines primarily because they are perturbation based, and require the knowledge of the index sample and variance of the features in the training data. Among the array of XAI methods, S-LIME \cite{zhou2021s}, uses the central limit theorem to obtain the optimal number of surrogate samples is a method with good performance, and hence a good baseline. For simulating the black-box prediction model, we used a Random Forest Classifier for all the classification tabular datasets. A summary of the dataset and prediction model statistics can be found in Table \ref{tab:dataset_description}. We used the open-source Scikit-Learn \cite{sklearn} implementation of the Random Forest classifier to simulate the black-box prediction models. 

\subsection{Numerical Results}
In this section, we numerically demonstrate the stability and the contrastive nature of variants of the CLIMAX algorithm. Our method works on data of different modalities such as tabular, text and image, and hence we showcase its performance for each modality.

\begin{table}
        \caption{Description of datasets.}
\resizebox{\columnwidth}{!}{%
  \begin{tabular}{|c|c|c|c|c|c|}
    \hline
    \textbf{Dataset} & $p$ & $n_{total}$ & Precision & Recall & AUC-ROC\\
    \hline
    \textbf{Breast Cancer} & 30 & 569 & 0.978 & 0.968 & 1.0\\
    \textbf{Parkinson's} & 22  & 175 & 0.96	& 1.0 & 0.857\\
    \textbf{Ionosphere} & 34  & 351 & 0.936  &	0.976 & 0.917\\
    \textbf{Diabetes} & 8 & 768 & 0.719 & 0.672 & 0.725\\
    \hline
    \end{tabular}%
}
    \label{tab:dataset_description}
\end{table}

\subsubsection{Stability in repeated explanations}
\begin{figure*}
        \begin{subfigure}[b]{0.25\textwidth}
                \centering
                \includegraphics[width=\linewidth]{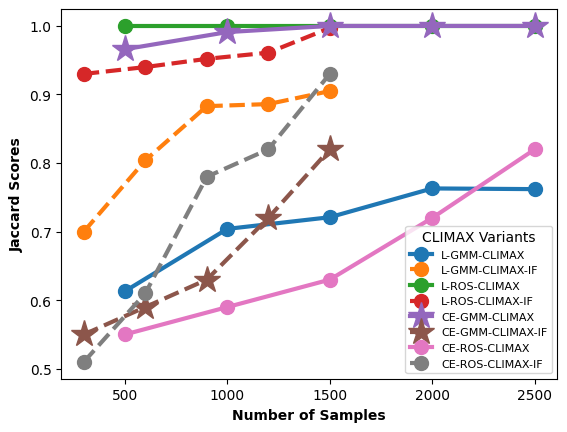}
                \caption{Breast Cancer Dataset}
                \label{fig:Breast Cancer Plot}
        \end{subfigure}%
        \begin{subfigure}[b]{0.25\textwidth}
                \centering
                \includegraphics[width=\linewidth]{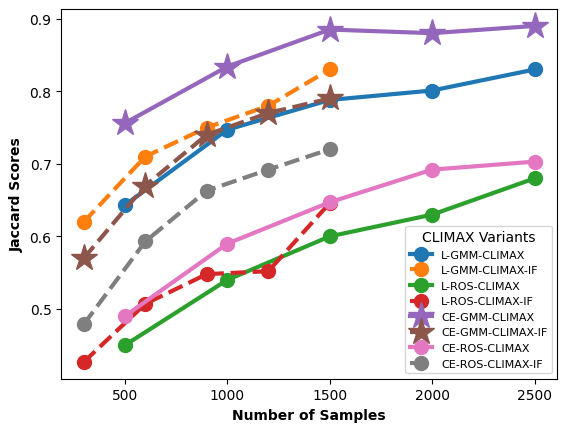}
                \caption{Parkinson's Dataset}
                \label{fig:Parkinson's Plot}
        \end{subfigure}%
        \begin{subfigure}[b]{0.25\textwidth}
                \centering
                \includegraphics[width=\linewidth]{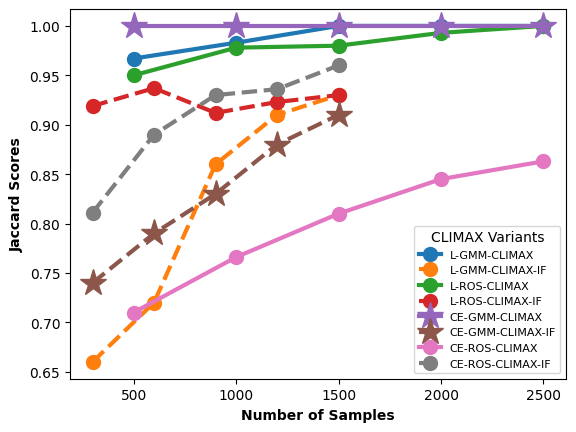}
                \caption{Ionosphere Dataset}
                \label{fig:Ionosphere Plot}
        \end{subfigure}%
        \begin{subfigure}[b]{0.25\textwidth}
                \centering
                \includegraphics[width=\linewidth]{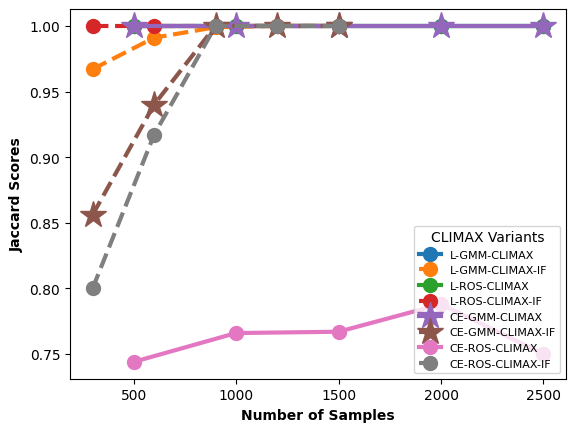}
                \caption{Diabetes Dataset}
                \label{fig:Diabetes Plot}
        \end{subfigure}

        \caption{Average Jaccard scores for $10$ randomly sampled test instances across varying surrogate samples, for variants of CLIMAX.}\label{fig:CLIMAX_Tabular}
\end{figure*}

\begin{figure*}
        \begin{subfigure}[b]{0.25\textwidth}
                \centering
                \includegraphics[width=\linewidth]{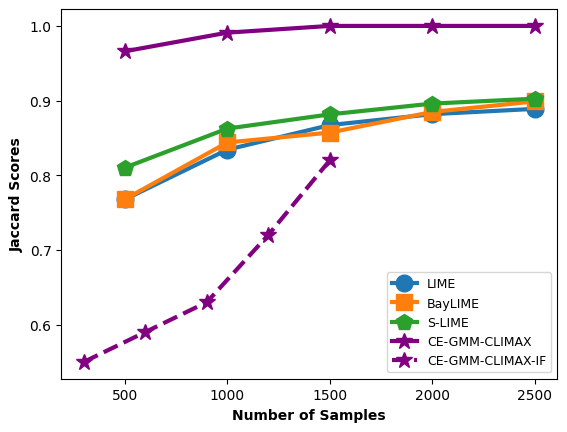}
                \caption{Breast Cancer Dataset}
                \label{fig:Breast Cancer Plot}
        \end{subfigure}%
        \begin{subfigure}[b]{0.25\textwidth}
                \centering
                \includegraphics[width=\linewidth]{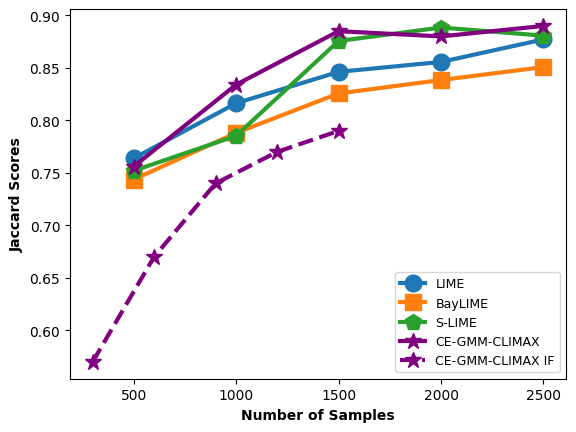}
                \caption{Parkinson's Dataset}
                \label{fig:Parkinson's Plot}
        \end{subfigure}%
        \begin{subfigure}[b]{0.25\textwidth}
                \centering
                \includegraphics[width=\linewidth]{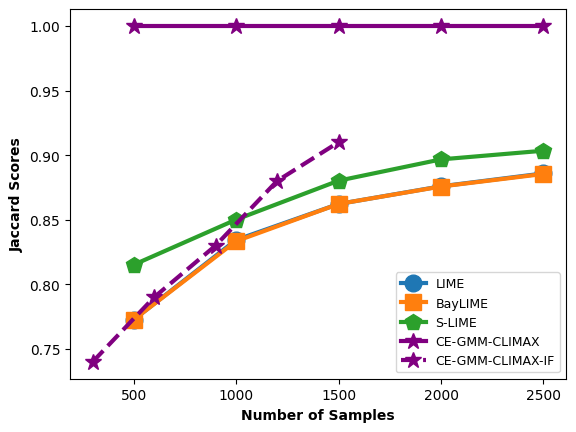}
                \caption{Ionosphere Dataset}
                \label{fig:Ionosphere Plot}
        \end{subfigure}%
        \begin{subfigure}[b]{0.25\textwidth}
                \centering
                \includegraphics[width=\linewidth]{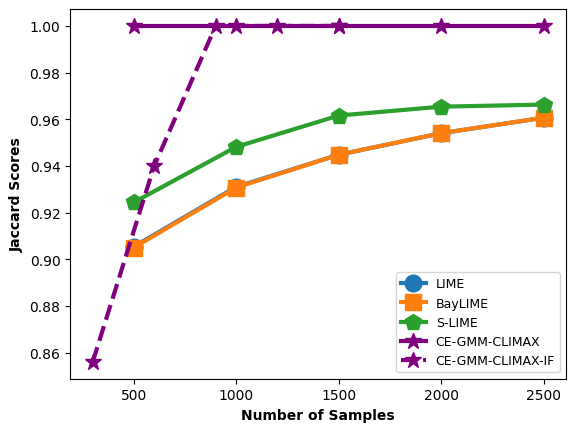}
                \caption{Diabetes Dataset}
                \label{fig:Diabetes Plot}
        \end{subfigure}

\caption{The average Jaccard scores for $10$ randomly sampled test instances across various numbers of surrogate samples of the best variant of CLIMAX (CE-GMM-CLIMAX) as compared to the three state-of-the-art methods.}
\label{fig:baseline_compare}
\end{figure*}

For evaluating the inconsistency in explanations over multiple runs, we execute CLIMAX, and the baselines using $500$, $1000$, $1500$, $2000$, and $2500$ surrogate samples and collected $20$ consecutive explanations for $10$ randomly selected index samples for each of the four datasets described in Table \ref{tab:dataset_description}.The Jaccard's distance(J) \cite{Unravel, DLIME} for measuring the consistency in  explanations across the $i$-th, and $j$-th run can be computed as follows:
    \begin{equation}
        J(X_i,X_j) = \frac{|X_i \cap X_j|}{|X_i \cup X_j|},
    \end{equation}
where $X_i$ and $X_j$ are sets consisting of top-5 features for iterations $i$ and $j$. Intuitively, it can be observed that $J(X_i, X_j)=1$ if $X_i$ and $X_j$ have the same features, and $J(X_i, X_j)=0$ if they have no common features. Thus, a consistent explainer module will have a relatively higher value of this metric than a relatively inconsistent explainer module. We averaged this metric over all possible combinations of iterations and the $10$ index samples. The results can be seen in Figure \ref{fig:CLIMAX_Tabular}. We average the values over $500$, $1000$, $1500$, $2000$, and $2500$. Across all datasets, incorporating the Cross-Entropy Loss along with sampling from a Gaussian Mixture Model (CE-GMM-CLIMAX) improved the model stability and fidelity to a large extent. Hence, we take only that method and it's influence subsampling counterpart forward and compare it with the other baseline methods in \ref{fig:baseline_compare}. It can be seen that for various sample sizes, CLIMAX outperforms both LIME, BayLIME, and S-LIME across all datasets. For S-LIME, we restrict the $n_{max}$ parameter to be $1.5$ times the size of the original number of samples. 
\subsubsection{CLIMAX Surrogate Data}
To evaluate the quality of the surrogate dataset generated by CLIMAX be it through GMM Sampling or Random Oversampling, we collected the surrogate data generated during the stability experiment for twenty samples from all our datasets and calculated the macro-precision and recall scores. CLIMAX improves these scores, through oversampling and then subsequently subsampling by influence. We depict this in Table 2, which explains how our explainer works with the surrogate data. 

It can be seen that the bootstrapping samples obtained using the procedure according to \cite{LIME,BayLIME} is highly imbalanced for all datasets (first row for each dataset). Further, we see that to a large extent the imbalance is removed using ROS and GMM under CLIMAX. Although ROS leads to an improvement in precision and recall scores, the information content in the data is the same as the bootstrapped samples. This necessitates a technique like GMM that also improves the quality of data. In some cases, the IF subsampled points lead to lower precision and recall scores. However, we believe that IF maintains the quality of data, and hence, lower precision-recall scores may not translate to poor explanation quality. 

\begin{figure}[h]
\centering
\includegraphics[width=0.4\textwidth]{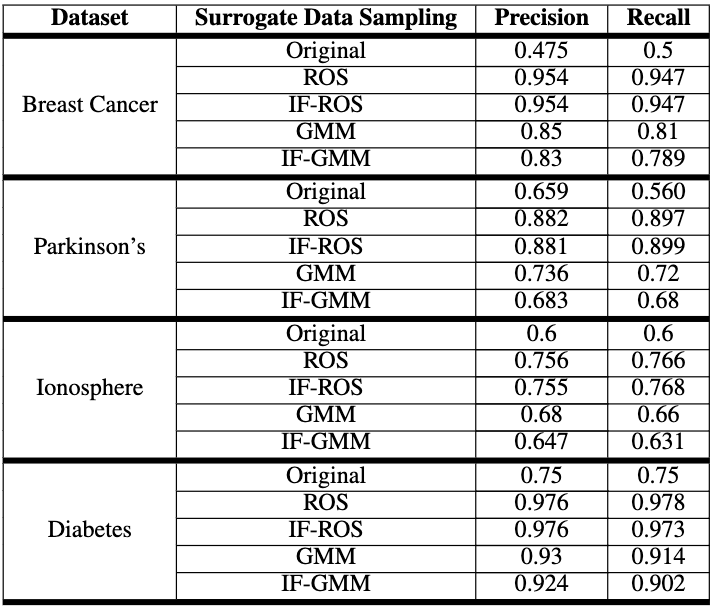}
\caption*{\textbf{Table 2}: Understanding the imbalance within the surrogate samples and randomly oversampling the data to get a fully balanced dataset.}
\label{fig:Table_orig_paper}
\end{figure}

\subsubsection{CLIMAX on Text Datasets}

To showcase CLIMAX's ability to provide robust textual explanations, we employed the information retrieval based tf–idf (term frequency–inverse document frequency) framework. We first extract features from the data using the tf-idf method. We train the black-box model and choose a test sample as the index sample. We compare our method with LIME in Figure \ref{fig:Text_Result_CLIMAX}.

We see that explanations of CLIMAX agree with LIME on many words (as in the highlighted text). However, the contrast in scores is large mainly because these explanations provide reasoning as to why one class is chosen instead of the other. The explanation of CLIMAX as compared to LIME on a large paragraph is provided using the 20 News Group dataset. Due to lack of space, we have moved this result to the supplementary.

\subsection{Climax on Image Dataset}

In the case of the image data, we first preprocess the data by using a popular segmentation algorithm called quickshift within the Scikit-image module \cite{LIME}. We depict the explanations provided by CLIMAX in Fig.~\ref{fig:better_image}. Due to space constraints, we provide a comparison between  LIME, CLIMAX and CEM in the supplementary. 

From Fig.~\ref{fig:better_image}, we see that an interesting benefit of contrastive explanations in CLIMAX is the possibility comparing explanations across classes. We show that regions in numerals provide explanations that are complementary to each other.  For example, similar to several works that investigate explanations for $3$ versus $5$ \cite{ignatiev2020relating}, we see that the explainer is sure about class $3$ due to the upper half, but neutral about the bottom half. Investigating digit $5$, we see that the explainer is neutral about the upper half, but neutral about the bottom half. This shows that CLIMAX is  not only contrastive within the same image, but consistent across images of different classes. We depict several such examples in the figure.

\begin{figure}
\centering
\begin{tabular} {cc}
\setlength\tabcolsep{0pt}

\includegraphics[width=0.22\textwidth]{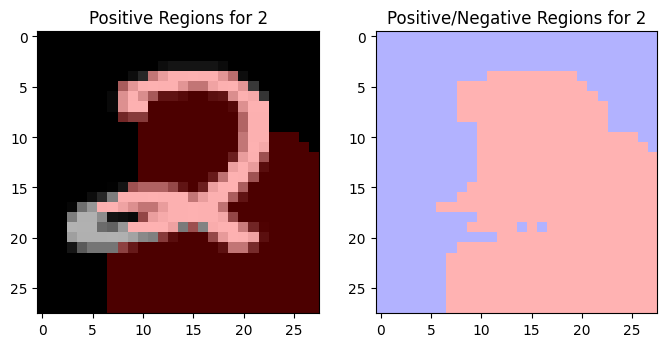} &
\includegraphics[width=0.22\textwidth]{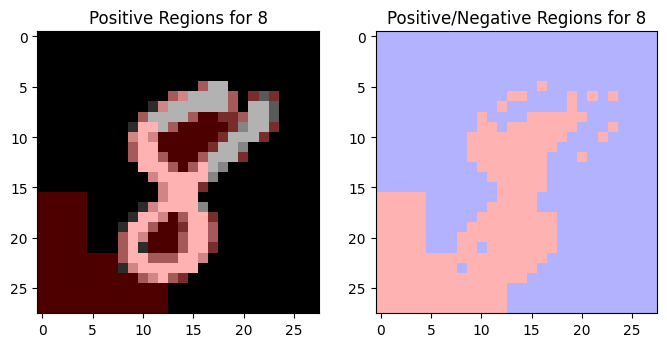} \\
\includegraphics[width=0.22\textwidth]{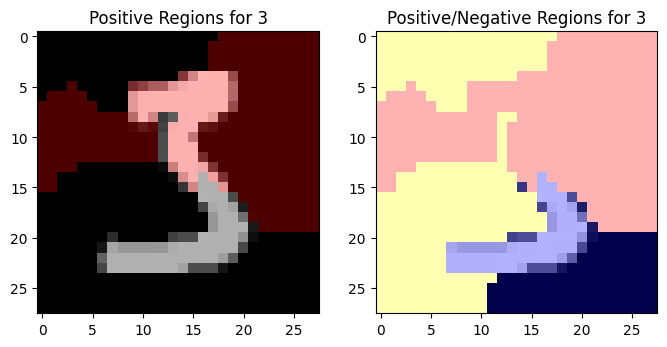} &
\includegraphics[width=0.22\textwidth]{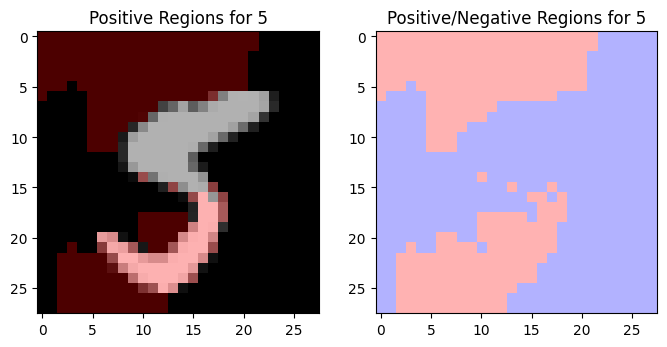} \\
\includegraphics[width=0.22\textwidth]{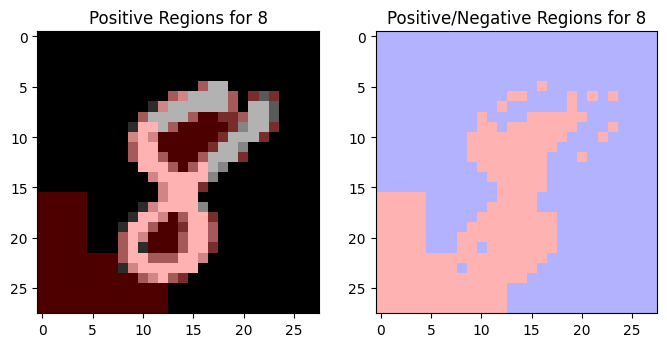} &
\includegraphics[width=0.22\textwidth]{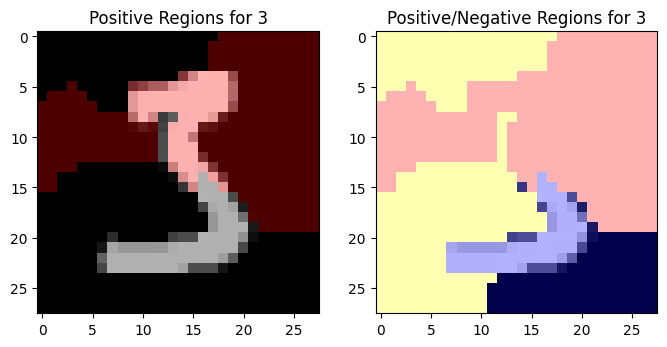}  \\
\includegraphics[width=0.22\textwidth]{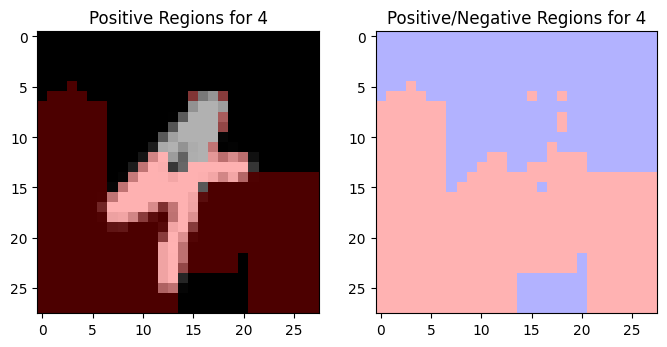} &
\includegraphics[width=0.22\textwidth]{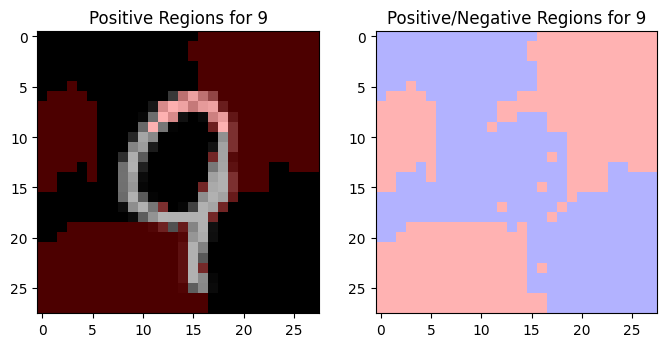}  \\
\end{tabular}
\caption{Understanding the contrastive nature of CLIMAX's image explanations for the MNIST Dataset. In the first row, the upper region of the digit 2 is marked positive and the same region for 8 is marked ambiguous by CLIMAX. Such a pattern is followed in all of these cases. What we mean by this is, the regions where the characteristics of a particular digit can be explicitly seen are given positive weightage. The same region is given an ambiguous state for other digits due to the same reason. Hence, the explanations help us in determining why a point lies in class $c_i$ and not in classes $c_{-i}$.}
\label{fig:better_image}
\end{figure}

\section{Conclusions and Future Work}
CLIMAX (\textbf{C}ontrastive \textbf{L}abel-aware \textbf{I}nfluence-based \textbf{M}odel-\textbf{A}gonostic \textbf{X}AI) is a perturbation based explainer, which exploits the classification boundary to provide contrastive results. CLIMAX perturbs the index sample to obtain surrogate samples by oversampling the instances of the minority class using random oversampling or GMMs, in order to obtain a balanced dataset. It also employs influence subsampling in order to reduce the sample complexity. Explanation scores are then provided using a logistic regression module, and we propose two variants in this direction. As compared to other perturbation-based techniques, CLIMAX explains why a point is in class $c_i$ and also provides information about why it is not in the remaining classes $c_{-i}$. CLIMAX gives access to the explainer to all class probabilities, helping in providing the contrastive scores. We observe that CLIMAX is able to produce more stable, faithful, and contrastive results as compared to LIME across different modalities of data. CLIMAX provides an important insight as the inherent task which is being explained is classification-based. Hence, CLIMAX is a well-rounded extension of LIME for black-box classifiers. In the future, we would like to provide uncertainty estimates for our explanations. This would help in checking the fidelity of the explanations better. 

\section{Additional Results and Discussions}
In this section, we demonstrate the efficacy of the proposed CLIMAX framework on publicly-available datasets. In particular, we are interested in establishing the contrastive capability and investigating the attributes such as stability (consistency) and sample efficiency in repeated explanations. We employ tabular, textual, and image datasets and consider different black-box models for an explanation.

\subsubsection{TSNE Plots: CLIMAX Surrogate Data}
To evaluate the quality of the surrogate dataset generated by CLIMAX using GMM Sampling, we plot to TSNE plots in the case of MNIST image dataset. In Fig.~\ref{fig:tsne2}, the index sample belonged to class 2 and we see that the GMM sampling occurs from 6 main clusters. This implies that the digit 2 is similar to six other digits, and that the sampling for the surrogate data is uniform. Similarly, in Fig.~\ref{fig:tsne3}, the index sample belonged to class 3 and we see that the sampling occurs from seven classes. If we look at both the remaining classes, which are dissimilar to the digits in question, are sampled uniformly from each of the similar clusters. Hence, we conclude that the surrogate sampling via GMM is not only effective, but also explainable. In comparison to methods that incorporate autoencoders and other black-box data generation mechanisms \cite{dhurandhar2018explanations,wang2022not}, we find our technique to be more transparent and trustworthy. 

\begin{figure}
\centering
\includegraphics[width=0.5\textwidth]{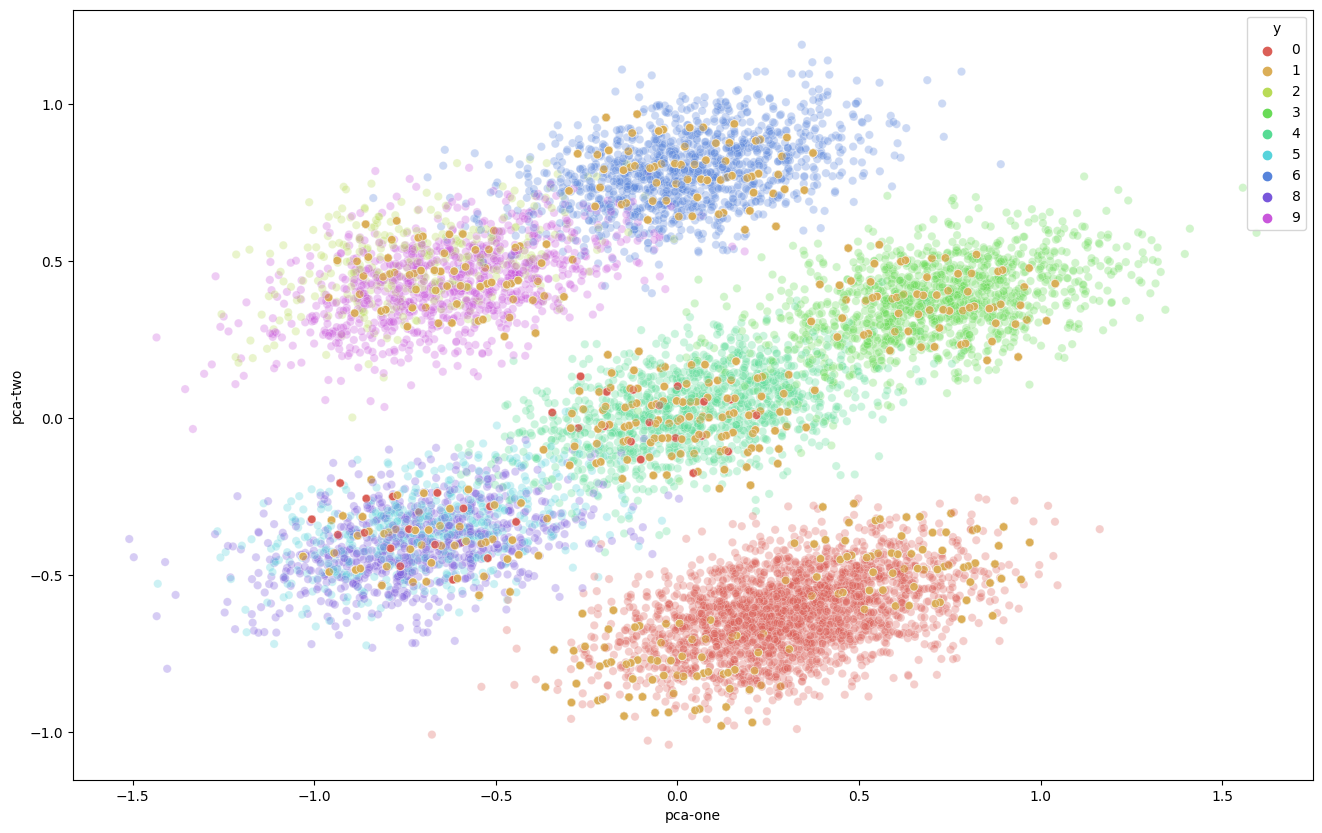}
\caption{This image shows the tSNE plots for the GMM sampling to cover for the oversampling of minority classes while sampling for the Digit 2. As we can see, the GMM samples well across the classes where the digits look similar to 2, and doesn't sample points from classes like 7, which are totally dissimilar. Hence, unlike Random Oversampling, where all classes get oversampled to the same sample size, GMM does this optimally. We reduce the imbalance, while maintaining the quality of the surrogate data, which is of utmost importance.}
\label{fig:tsne2}
\end{figure}

\begin{figure}
\centering
\includegraphics[width=0.5\textwidth]{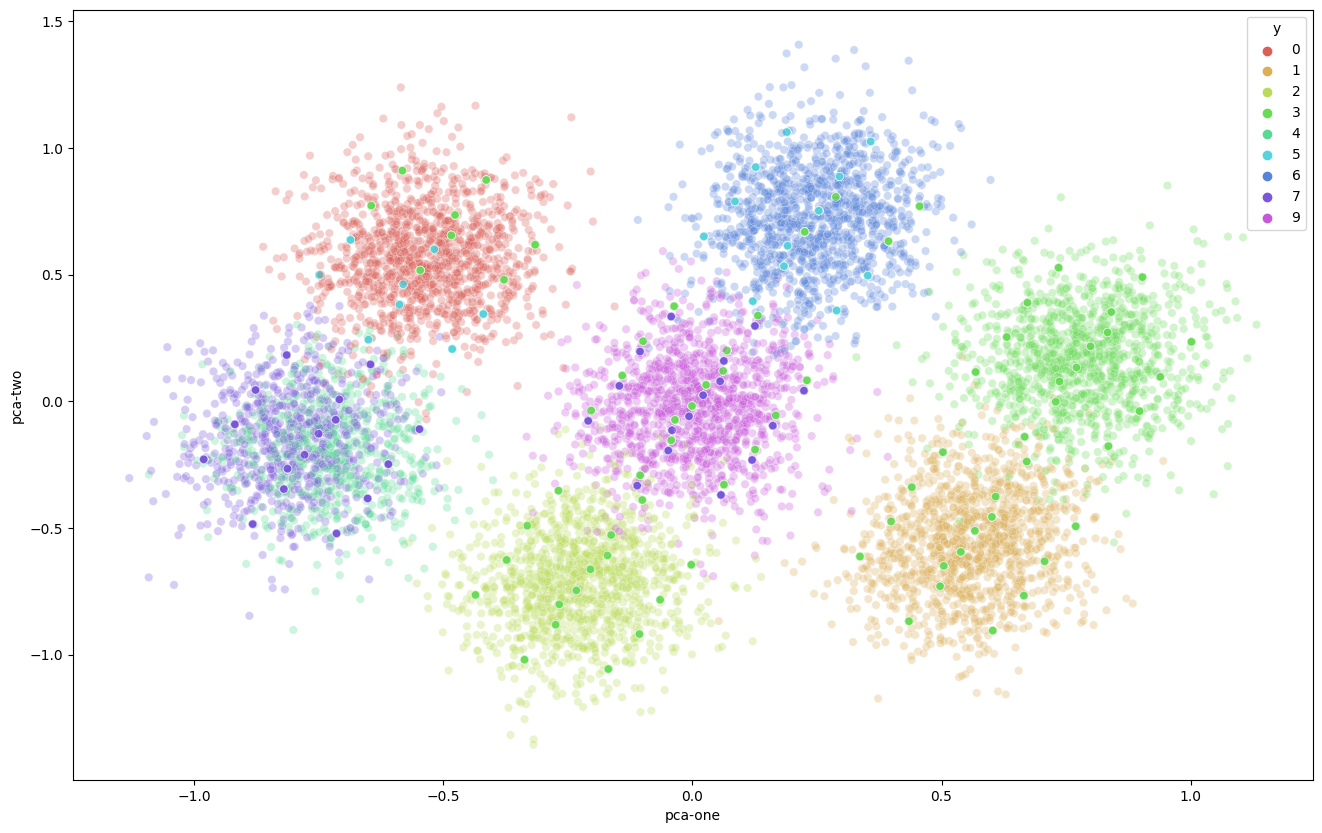}
\caption{This image shows the tSNE plots for the GMM sampling to cover for the oversampling of minority classes while sampling for the Digit 3. As shown in \ref{fig:tsne2}, GMM samples smartly from the given $y$ values. The only difference is that, for 2, the number of classes the GMM considered for oversampling were 6 and for 3 it creates 7 clusters. This is because 3, as a digit, is similar to 7 other digits from different angles.}
\label{fig:tsne3}
\end{figure}

\subsubsection{CLIMAX on Text Datasets}

To showcase CLIMAX's ability to provide robust textual explanations, we employed the  information retrieval based tf–idf (term frequency–inverse document frequency) framework.We first extract features from the data using the tf-idf method. We compare our method with LIME in Figure \ref{fig:CLIMAX_vs_LIME_20NewsGroup}.
We see that even on longer paragraphs of text, CLIMAX maintains its contrastive capability, as compared to LIME. 

\begin{figure*}
        \begin{subfigure}[b]{0.33\textwidth}
                \centering
                \includegraphics[width=.85\linewidth]{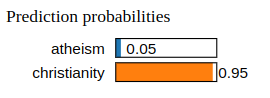}
                \caption{Class probabilities of the index sample.}
        \end{subfigure}%
        \begin{subfigure}[b]{0.33\textwidth}
                \centering
                \includegraphics[width=.85\linewidth]{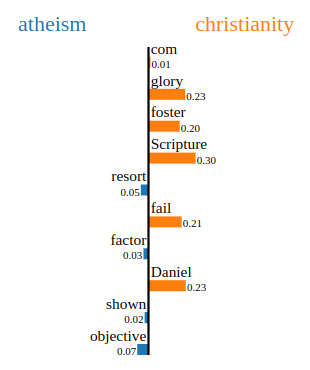}
                \caption{CLIMAX Explanation}
        \end{subfigure}%
        \begin{subfigure}[b]{0.33\textwidth}
                \centering
                \includegraphics[width=.85\linewidth]{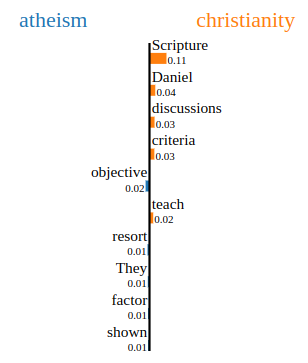}
                \caption{LIME Explanation}
        \end{subfigure}
        \begin{subfigure}[b]{\textwidth}
        \centering
        \includegraphics[width=.85\linewidth]{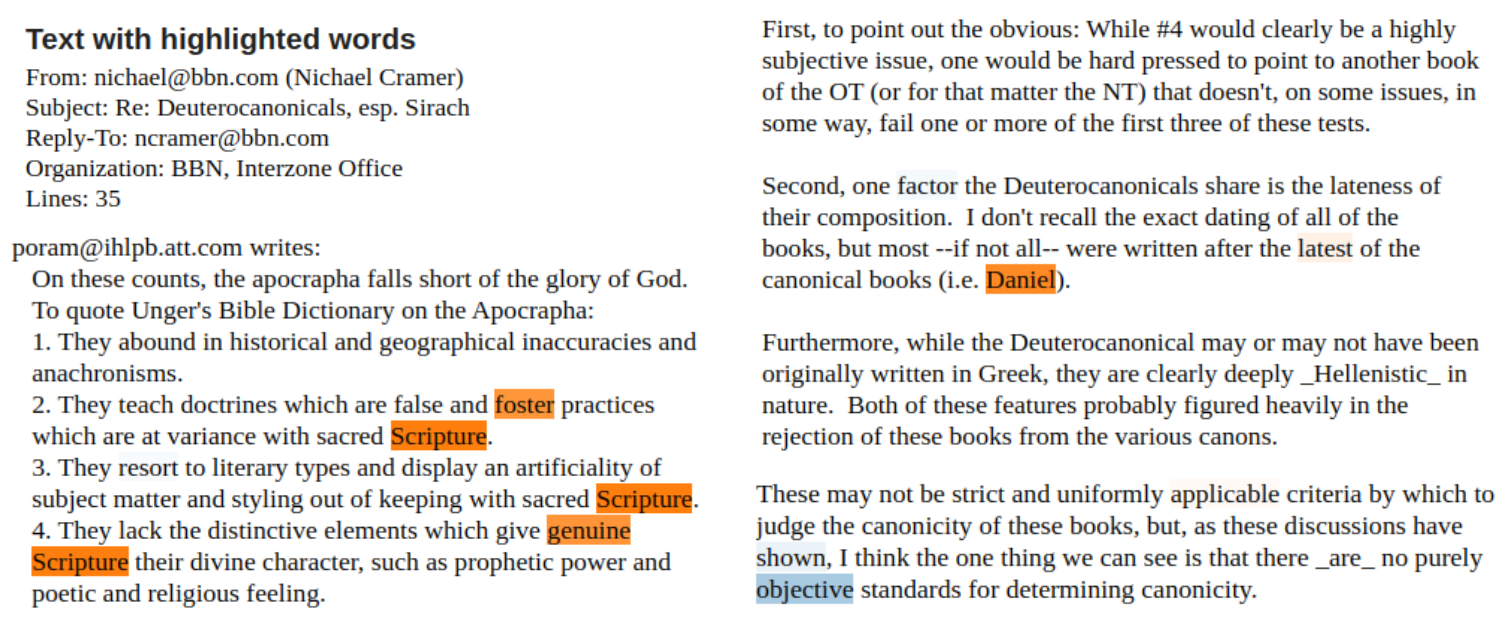}
        \caption{The instance for which the explanations have been provided.}
        \end{subfigure}
        \caption{Explanations for CLIMAX and LIME for the same instance of the 20 Newsgroups dataset.}
        \label{fig:CLIMAX_vs_LIME_20NewsGroup}
\end{figure*}

\begin{figure*}
\begin{tabular}{ccccccc}
    \toprule
    \bfseries LIME Explanation & 
    \bfseries LIME Region & 
    \bfseries CLIMAX Explanation &
    \bfseries CLIMAX Region &
    \bfseries CEM Image &
    \bfseries CEM PP Region &
    \bfseries CEM PN Region \\
    \hline
\adjustimage{height=2cm,valign=m}{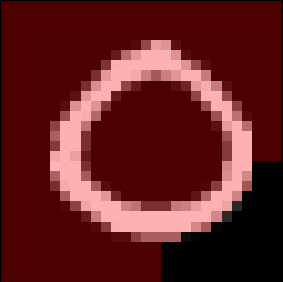} &
\adjustimage{height=2cm,valign=m}{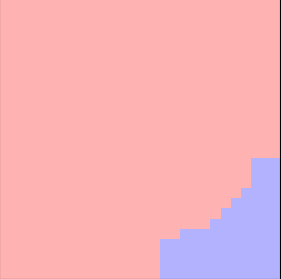} &
\adjustimage{height=2cm,valign=m}{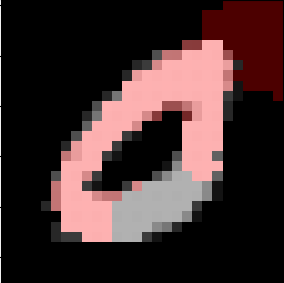} &
\adjustimage{height=2cm,valign=m}{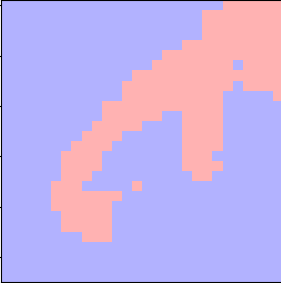} &
\adjustimage{height=2cm,valign=m}{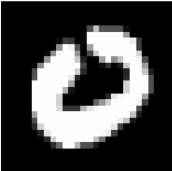} &
\adjustimage{height=2cm,valign=m}{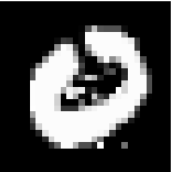} &
\adjustimage{height=2cm,valign=m}{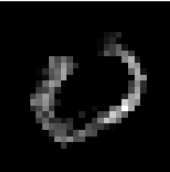} \\
    
\adjustimage{height=2cm,valign=m}{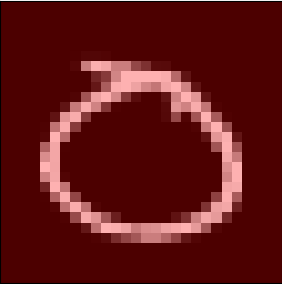} &
\adjustimage{height=2cm,valign=m}{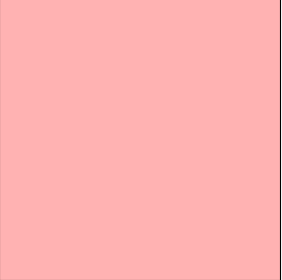} &
\adjustimage{height=2cm,valign=m}{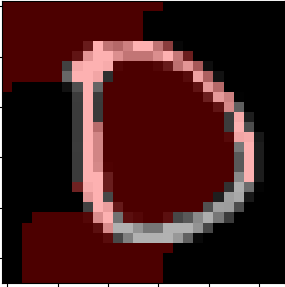} &
\adjustimage{height=2cm,valign=m}{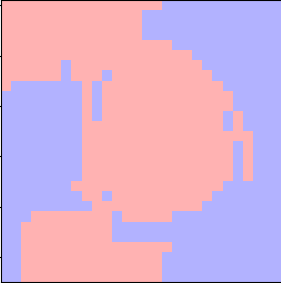} &
\adjustimage{height=2cm,valign=m}{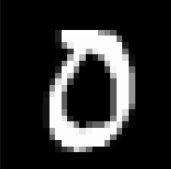} &
\adjustimage{height=2cm,valign=m}{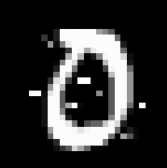} &
\adjustimage{height=2cm,valign=m}{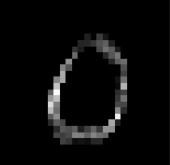} \\

\adjustimage{height=2cm,valign=m}{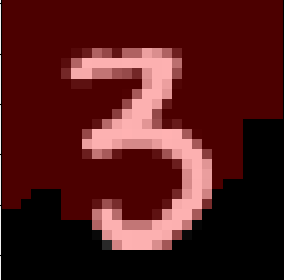} &
\adjustimage{height=2cm,valign=m}{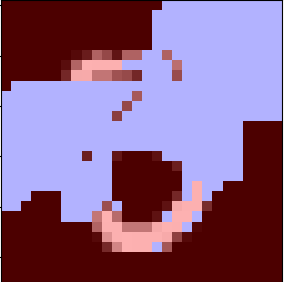} &
\adjustimage{height=2cm,valign=m}{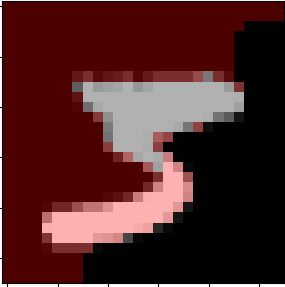} &
\adjustimage{height=2cm,valign=m}{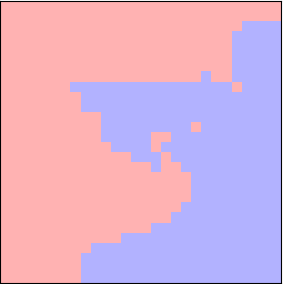} &
\adjustimage{height=2cm,valign=m}{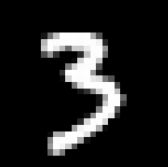} &
\adjustimage{height=2cm,valign=m}{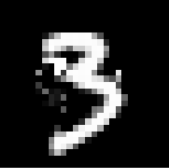} &
\adjustimage{height=2cm,valign=m}{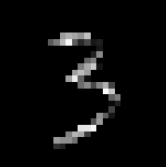} \\

\adjustimage{height=2cm,valign=m}{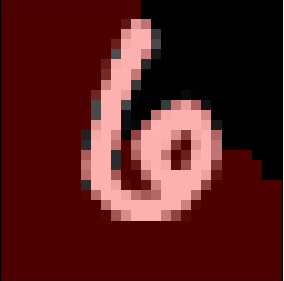} &
\adjustimage{height=2cm,valign=m}{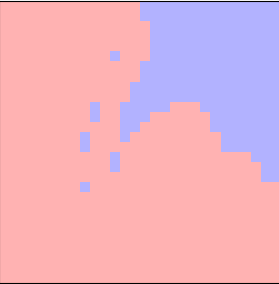} &
\adjustimage{height=2cm,valign=m}{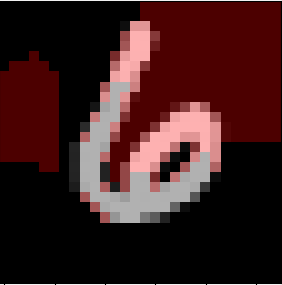} &
\adjustimage{height=2cm,valign=m}{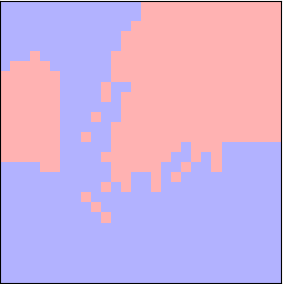} &
\adjustimage{height=2cm,valign=m}{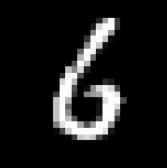} &
\adjustimage{height=2cm,valign=m}{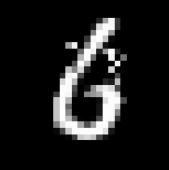} &
\adjustimage{height=2cm,valign=m}{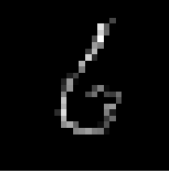} \\

\adjustimage{height=2cm,valign=m}{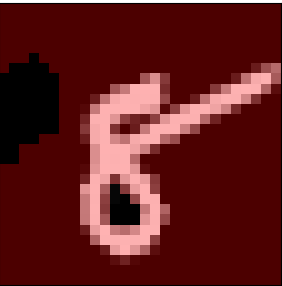} &
\adjustimage{height=2cm,valign=m}{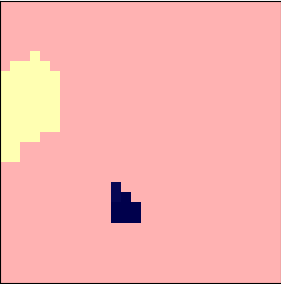} &
\adjustimage{height=2cm,valign=m}{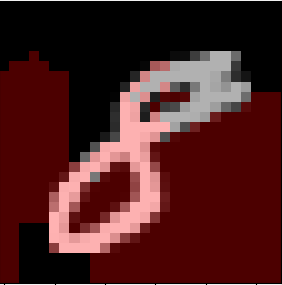} &
\adjustimage{height=2cm,valign=m}{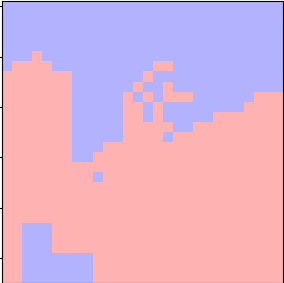} &
\adjustimage{height=2cm,valign=m}{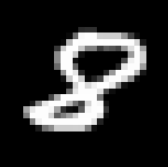} &
\adjustimage{height=2cm,valign=m}{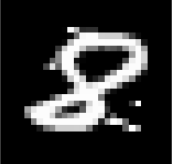} &
\adjustimage{height=2cm,valign=m}{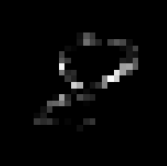} \\

\adjustimage{height=2cm,valign=m}{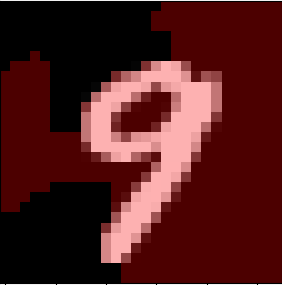} &
\adjustimage{height=2cm,valign=m}{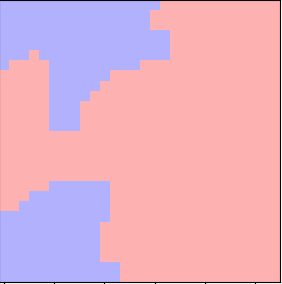} &
\adjustimage{height=2cm,valign=m}{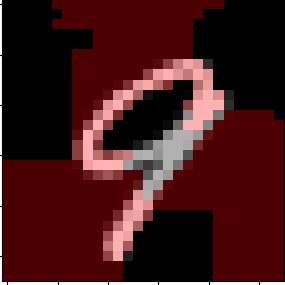} &
\adjustimage{height=2cm,valign=m}{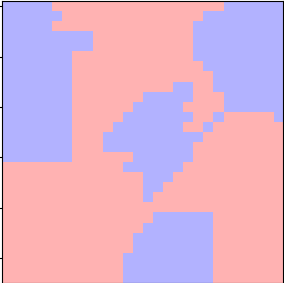} &
\adjustimage{height=2cm,valign=m}{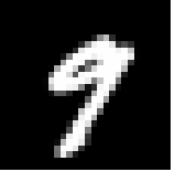} &
\adjustimage{height=2cm,valign=m}{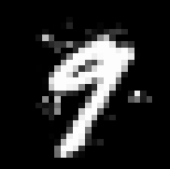} &
\adjustimage{height=2cm,valign=m}{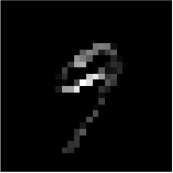} \\
\label{tab:key}
\end{tabular}
\caption{We compare CLIMAX with  LIME and the Contrastive Explanations Method (CEM). We do so as CLIMAX and CEM both aim to provide more contrastive results. If we look at the explanation masks of LIME, we can see that it attributes an unnecessarily large region for an explanation, even when it is provided with a larger surrogate sample size. In the case of CEM, the Pertinent Positive (PP) regions and the Pertinent Negative (PN) regions do bring out a contrastive flavour, but CLIMAX shows ambiguity in regions where multiple digits look similar, making it more visually trustworthy.}
\label{fig:fig_key}
\end{figure*}

\subsection{CLIMAX for images Vs CEM \cite{dhurandhar2018explanations}}
In the case of CEM, the classification boundary is exploited well in terms of the regions that are Pertinently Positive and Pertinently Negative. However, the ambiguity in the sub-parts of an image due to overlap in the PP and the PN regions makes the classification of a digit uncertain. For instance, common regions in the digit $0$ have been marked as pertinent positive and pertinent negative. It is not clear how the en-user is supposed to interpret these areas. In particular, joint analysis of pertinent positive and pertinent negative regions are not possible. Moreover, comparison across different digits is also not possible. CLIMAX does not face such challenges. On one hand, within the same digit, it clearly marks the regions that it is certain (in pink for all digits) and uncertain (in grey for all digits). Furthermore, we can also analyse across digits, where if one area is marked positive for a certain digit, then that same area would be marked neutral for many other digits. This nature is visible across all CLIMAX explanations. 

\subsection{CLIMAX for images vs LIME\cite{LIME}}
We compare CLIMAX with LIME, the most popular post-hoc explainable AI method, which set the foundation for such methods in Fig.~\ref{fig:fig_key}. We see that often LIME is not able to distinguish between regions encapsulated within a digit and the digit boundary itself. This is mainly because LIME does not take as input all class probabilities, and employ any decision boundary aware mechanism. The contrastive nature of the explanations is evident here as well, where CLIMAX tends to indicate neutral regions which it is not sure about. However, LIME does not capture such contrast. 

\bibliography{ecai}
\end{document}